\documentclass[twocolumn]{article}
\usepackage{lipsum} % For placeholder text
\usepackage{geometry} % For adjusting margins
\usepackage{graphicx} % For including images
\usepackage{amsmath,amssymb} % For mathematical symbols 
\usepackage{mathrsfs}

\usepackage{graphicx} % Required for inserting images
\usepackage{amsmath}
\usepackage{amsfonts} 
\usepackage{natbib}
\usepackage{bm}
\usepackage[colorlinks=true, citecolor=blue, linkcolor=blue, urlcolor=blue]{hyperref}

\usepackage{amsthm}
\usepackage{xcolor}
\usepackage{comment}
\usepackage{bibentry}
\usepackage[flushleft]{threeparttable}
\usepackage{booktabs}
\usepackage{algorithm}
\usepackage[noend]{algorithmic}
\theoremstyle{definition}
\newtheorem{theorem}{Theorem}[section]

\newtheorem{proposition}[theorem]{Proposition}

\newtheorem{remark}[theorem]{Remark}

\newtheorem{definition}[theorem]{Definition}
\newtheorem{assumption}[theorem]{Assumption}
\numberwithin{equation}{section}
\newcommand{\hatd}[1]{{}}
\newcommand{\JOBID}[1]{241009-experiment_1-exponential_decay_kernel}
\newcommand{\JOBIDONLINE}[1]{241009-experiment_2-exponential_decay_kernel_decrease}

\usepackage{tikz}
\usepackage{caption}
\usepackage{subcaption}
\usetikzlibrary{arrows,positioning} 
\tikzset{
    >=stealth',
    punkt/.style={
           rectangle,
           rounded corners,
           draw=black, very thick,
           text width=6.5em,
           minimum height=2em,
           text centered},
    pil/.style={
           ->,
           thick,
           shorten <=2pt,
           shorten >=2pt,}
}

\title{Deep Reinforcement Learning for Online Optimal Execution Strategies}

\author{
  Alessandro Micheli\thanks{Joint first authors.} \\
  Imperial College London \\
  London, UK \\
  \texttt{a.micheli19@imperial.ac.uk}
  \and
  Mélodie Monod\footnotemark[1] \\
  Imperial College London \\
  London, UK \\
  \texttt{melodie.monod18@imperial.ac.uk}
}

\date{}

\begin{document}

\maketitle

\begin{abstract}
\sloppy  This paper tackles the challenge of learning non-Markovian optimal execution strategies in dynamic financial markets. 
We introduce a novel actor-critic algorithm based on Deep Deterministic Policy Gradient (DDPG) to address this issue, with a focus on transient price impact modeled by a general decay kernel. 
Through numerical experiments with various decay kernels, we show that our algorithm successfully approximates the optimal execution strategy. Additionally, the proposed algorithm demonstrates adaptability to evolving market conditions, where parameters fluctuate over time.
Our findings also show that modern reinforcement learning algorithms can provide a solution that reduces the need for frequent and inefficient human intervention in optimal execution tasks.
\end{abstract}

%
%% INTRODUCTION %%

\section{Introduction}
% Problem statement

Optimal execution, a critical challenge in algorithmic trading, involves finding a strategy to execute sizable orders while minimizing market impact. 
Traditionally, researchers have relied on parametric models of \textit{price impact} to determine optimal execution strategies in closed-form (e.g. see \cite{Obizhaeva2013}). 
Despite its widespread adoption, this approach presents certain practical drawbacks.
Firstly, it comes with the risk of adopting potentially inaccurate assumptions about market dynamics embedded within price impact models. 
Secondly, it is inefficient as price models require periodic recalibration using the latest financial data to adapt to the ever-changing nature of markets. Doing so demands significant manual effort and oversight, potentially leading to increased transaction costs due to outdated calibrations that misestimate price impact.
As such, the traditional approach remains heavily reliant on human interventions and heuristic adjustments.

% Traditionally, \textit{price impact models} with inherent market dynamics assumptions have been used to derive analytical (i.e., closed-form) optimal execution strategies 

% The periodic recalibration of price impact models using historical data inherently limits their ability to adapt to the dynamic nature of modern financial markets. 

% This reliance on past data necessitates significant manual effort and oversight to maintain accurate calibrations, as outdated estimations can lead to increased transaction costs due to inaccurate market impact predictions.

% First, the periodic recalibration of price impact models' parameter on the most recent financial data does not allow to accurately adjust to the dynamic nature of modern financial markets. This requires significant manual effort and oversight and increases transaction costs since outdated calibrations inaccurately estimate price impact.

% First, this approach ignores the dynamic nature of modern financial markets as parameters estimated from historical data necessitate periodic recalibration. Outdated calibrations can lead to higher transaction costs due to inaccurate price impact estimates and require significant manual effort and oversight for continuous recalibration (Inefficiency). 

% Literature review
\textit{Reinforcement Learning} (RL) methods~\citep{russell2016artificial} are emerging as a powerful alternative to address the challenges of optimal execution. 
% RL simulates an agent interacting with the complexities of order book dynamics and market liquidity conditions. The agent executes trades, observes the resulting market impacts (rewards), and learns to make better decisions over time. 
%\textit{Model-free} and \textit{online} RL algorithms hold promise in overcoming the limitations of traditional methods.
%that rely on pre-defined models and periodic recalibration.
% Specifically, RL's model-free nature eliminates the need for potentially inaccurate assumptions embedded within traditional price impact models (Model Dependence). Furthermore, online learning allows the agent to continuously adapt to changing market conditions, addressing the issue of Inefficiency caused by outdated calibrations.
The introduction of the \textit{Double Deep Q-Network} (DDQN) algorithm marked a significant step forward, demonstrating the effectiveness of RL methods for optimal execution in a \textit{model-free} context~\citep{Ning2021} and in an \textit{online} time-varying liquidity environment~\citep{macri_lillo}. However, current state-of-the-art RL methods for optimal execution still have limitations that hinder their real-world application. 
First, due to the discrete nature of the action space in the DDQN algorithm, available execution strategies are limited in flexibility and complexity.
Furthermore, to the best of our knowledge, research has focused on market models with temporary and permanent impact components, which simplifies the problem by making it Markovian.
%Other recent applications of DDQN for optimal execution include~\citep{Dabeacuterius2019, Cartea2021}. 
%However, this approach has two major drawbacks.
%First, the DDQN algorithm restricts the action space to discrete trades.
%Second, this approach was designed for market models with temporary and permanent impact components, simplifying the problem by making it Markovian. 
However, market microstructure research suggests that market impact is \textit{transient}, meaning it decays over time. This reality renders the Markovian assumption inapplicable to optimal execution problems with transient impact (see for example the analysis of \citet{AbiJaber2022}), making approaches from previous works \citep{Ning2021, macri_lillo} unsuitable for real-world use.

%  First, its action space is restricted to discrete choices. Second, they were designed to perform on market models where price impact has a temporary and a permanent component. While this framework simplifies the problem by making it Markovian, recent market microstructure research suggests that market impact is transient, that is, it decays over time. Due to the time-dependent nature of transient impact, the optimal execution problem becomes path-dependent. 
% Therefore there is an ongoing need for efficient model-free solutions with continuous action spaces for optimal execution in the present of transient price impact.

% There are other applications of DDQN for optimal execution in the recent literature, including~\citep{Dabeacuterius2019, Cartea2021}.
%  Consequently, the evaluated strategies are often compared to the Time-Weighted Average Price (TWAP) strategy, which involves selling a fixed amount of shares at each time step.rendering TWAP suboptimal in this context.
% Secondly, a major limitation of DDQN is its reliance on discrete action spaces. This restricts the agent's ability to precisely execute trades, as the actual space of trade sizes is continuous. 

% Other RL algorithms \cite{neuman_offline,statistical_learning_neuman,Leal2022} address this issue by allowing continuous actions, they are limited to offline learning. This highlights the ongoing need for efficient, model-free solutions with continuous action spaces for optimal execution in algorithmic trading.

% Our approach
Motivated by these limitations, this work proposes a novel model-free algorithm with a continuous action space for optimal execution in the presence of transient price impact. We introduce a novel approach for optimal execution that leverages the \textit{Deep Deterministic Policy Gradient} (DDPG) algorithm~\citep{deep_dpg_lillicrap}. We address the limitations of existing methods through three key contributions.
Firstly, we leverage the DDPG algorithm's ability to handle continuous action spaces, allowing for decisions beyond simple binary buy/sell actions. 
To further enhance convergence, we introduce an auxiliary Q-function specifically designed for our optimal execution problem.
Secondly, unlike traditional model-based methods, our DDPG approach is model-free. This eliminates the need for strict assumptions about the market's price impact model. The algorithm's input space effectively adapts to non-Markovian problems with transient price impact governed by a general decay kernel.
Lastly, our DDPG algorithm is capable of online learning within and across episodes. This enables real-time adaptation of the optimal strategy in dynamic markets, removing the need for frequent recalibration.
The code to replicate our numerical experiments is freely available at~\url{https://github.com/ImperialCollegeLondon/deep_rl_liquidation}.

\section{Problem formulation}
\label{sec-problem-formulation}
\subsection{Market Impact Model} \label{sec:market_impact_model}
We review the market impact model of \cite{Alfonsi2012} for an agent trading in a risky asset. Let $T\in\mathbb{N}$ be a deterministic time horizon and fix a filtered probability space~\sloppy$(\Omega, \mathcal{F}, (\mathcal{F}_{t})_{t\in[0,T]}, \mathbb{P})$ satisfying the usual conditions and supporting a standard Brownian motion $(W_{t}^{0})_{t\in[0,T]}$. 
We suppose that the agent can trade at the times of an equidistant time grid $\mathbb{T}= \{t_{0},\ldots,t_{N}\}$ where $N\in\mathbb{N}$ and $0= t_{0} < t_{1} <\ldots < t_{N}=T$. A trade in the risky asset at time $t_{k}$ is described by a $\mathcal{F}_{t_{k}}$-measurable random variable $\xi_{t_{k}}$, where positive values denote buys and negative values denote sells.  The agent's goal is to choose an admissible strategy to liquidate a given initial portfolio\footnote{The case of a negative initial portfolio is equivalent up to minor modifications.} $X_{0}>0$.
\begin{definition}
\label{def-admissible-strategy}
An \textit{admissible strategy} is a sequence $\xi = (\xi_{t_{0}},\xi_{t_{1}},\ldots, \xi_{t_{N}})$ of bounded random variables such that each $\xi_{t_{k}}$ is $\mathcal{F}_{t_{k}}$-measurable. For a given initial portfolio $X_{0}$  and time grid $\mathbb{T}$, the set $\mathscr{X}(\mathbb{T},X_{0})$ is the set of admissible strategies $\xi = (\xi_{t_{0}},\ldots, \xi_{t_{N}})$ such that $$X_{0} + \sum_{k=0}^{N}\xi_{t_{k}}=0\quad \mathbb{P}-\text{a.s.}$$
%\begin{equation}
% \label{eq-admissible-strategies-definition}
% \begin{aligned}
% \mathscr{X}(\mathbb{T},X_{0}) := \left\{\xi = (\xi_{t_{0}},\ldots, \xi_{t_{N}})\  |\ \xi \text{ is admissible and } X_{0} - \sum_{k=0}^{N}\xi_{t_{k}}=0 \quad  \mathbb{P}-\text{a.s.}\right\}.
% \end{aligned}
% \end{equation}
\end{definition}
\noindent
Given an admissible strategy $\xi \in\mathscr{X}(\mathbb{T},X_{0})$, the outstanding number of shares the agent holds at time $t\in[0,T]$ is given by
\begin{equation}
\label{eq-agent-inventory}
X_{t}^{\xi} = X_{0} + \sum_{t_{k}<t}\xi_{t_{k}},
\end{equation}
where $X_{0}^{\xi} = X_{0}$ and $X_{T+}^{\xi} = 0$. 
We assume that the actual asset price will depend on the strategy chosen by the agent. As long as the agent is not active, asset prices are determined by the actions of the other market participants and are described by the  \textit{unaffected price process} $\left(P_{t}^{0}\right)_{t\in[0,T]}$. 
We take the unaffected price process to be
\begin{equation*}
P^{0}_{t} = p_{0} + \sigma_{W} W_{t},\quad P_{0}^{0} = p_{0},
\end{equation*}
where $(W_{t})_{t\in[0,T]}$ represents a standard Brownian motion and $p_{0}, \sigma_{W}>0$. 
When the admissible strategy $\xi\in\mathscr{X}(\mathbb{T},X_{0})$ is used by the agent, the risky asset execution  price $\left(P_{t}^{\xi}\right)_{t\in[0,T]}$ is given by
\begin{equation} \label{eq:execution_price}
    P_{t}^{\xi} := P^{0}_{t} +\sum_{t_{k}<t} G(t-t_{k}) \xi_{t_k}, \quad t\in[0,T],
\end{equation}
for a measurable function $G:[0,\infty)\to (0,\infty)$, the \textit{decay kernel}. The effect of the agent's trades on prices is \textit{transient}, meaning it decays over time. This decay is characterized by the decay kernel which reflects the resilience of price impact over time (see e.g. \citet{Gatheral2012}).
Henceforth, we will work under the following assumption (see also Remark \ref{remark-strictly-convex}). 
\begin{assumption}[Strictly Convex Kernel] \label{assumption-convex-kernel} 
We assume that the decay kernel $G(\cdot)$ is a strictly convex function.
\end{assumption}

\subsection{Optimal Execution Problem} \label{sec:optimal_execution_problem}
Let us now define the profits of an admissible strategy $\xi$. The execution of the $k^{\text{th}}$ order $\xi_{t_{k}}$ will shift the price from $P_{t_{k}}^{\xi}$ to $P_{t_{k}+}^{\xi} = P_{t_{k}}^{\xi} + \xi_{t_{k}} G(0)$. This is equivalent to assuming a ``block shape'' limit order book, where the number of shares available at each price $y$ is given by $G(0)^{-1}dy$ (see \citet[Section 2]{Alfonsi2009}). The costs incurred from executing the amount of shares $\xi_{t_{k}}$ of the risky asset at time $t_k$ are given by 
\begin{equation} \label{eq-returns} \int^{P_{t_{k}+}^{\xi}}_{P_{t_{k}}^{\xi}} yG(0)^{-1}dy = \frac{\left(P_{t_{k}+}^{\xi}\right)^{2}-\left(P_{t_{k}}^{\xi}\right)^{2}}{2G(0)}.
\end{equation}
The expected execution profits over the time period $[0,T]$ are given by
\begin{equation}
\label{eq-expected-costs}
V^{\xi} := -\frac{1}{2G(0)} \mathbb{E}\left[\sum_{k=0}^{N}\left(\left(P_{t_{k}+}^{\xi}\right)^{2}-\left(P_{t_{k}}^{\xi}\right)^{2}\right)\right],
\end{equation}
where the minus sign is introduced by switching from costs to profits. 
Starting from an initial portfolio $X_{0}$, the agent’s goal is to find the optimal strategy $\xi^{*}$ that maximizes the agent's objective: 
\begin{equation}
\label{eq-optimal-execution-problem}
V^{\xi^{*}} = \sup_{\xi\in\mathscr{X}(\mathbb{T},X_{0})} V^{\xi}.
\end{equation}
As shown in \citet[Section 3]{Alfonsi2012}, under Assumption \ref{assumption-convex-kernel} there exists a unique optimal strategy to the agent's optimal execution problem of \eqref{eq-optimal-execution-problem} and it is given by 
\begin{equation} \label{eq:optimal_strategy}
    \xi^{*}_{\text{AAS}} = -\frac{X_{0}}{\bm{1}M^{-1}\bm{1}} M^{-1}\bm{1},
\end{equation}
where $M^{-1}$ is the inverse of the matrix with entries $M_{ij} = G(|t_{i}-t_{j}|)$ and $\bm{1}$ is the $(N+1)$-dimensional vector whose components are all 1. 
\begin{remark}
\label{remark-strictly-convex}
Assumption \ref{assumption-convex-kernel} guarantees the absence of transaction-triggered price manipulation strategies. As shown in \citet[Section 5]{Alfonsi2012}, it is also a sufficient condition to ensure that all trades in an optimal strategy have the same sign. This forbids any pathologically alternation of buy and sell trades once trading a nontrivial amount of shares is required.
\end{remark}

%
%% METHODS %%

%\section{Approximating Optimal Execution Strategies with Deep Reinforcement Learning}
\section{Reinforcement Learning}
\label{sec-reinforcement-learning}
RL's ability to learn through trial and reward in dynamic environments makes it ideal for the optimal execution problem we formulated in Section \ref{sec:optimal_execution_problem}. In this section, we first introduce the RL framework and then we turn to the specific details of our algorithm.

\subsection{Markov Decision Processes}
We formalize the optimal execution problem within the language of RL by using the concept of a Markov Decision Process (MDP)~\citep{Puterman1994}. Specifically, a MDP is parametrised by a 4-tuple 
$\langle \mathcal{S},\mathcal{A}, \mathscr{R}, P\rangle$ where $\mathcal{S}$ is the state space, $\mathcal{A}$ is the action space, $\mathscr{R}$ is the reward function and $P$ is the transition probability function. 

The state of the market at each time $t\in\mathbb{T}$ is denoted by $s_{t}$. It captures relevant information about the current market conditions that influence the agent's decision-making process.  It is represented by a 4-tuple as:
\begin{equation*}
s_{t}=\left(t,X_{t}^{\xi}, \xi_{0:(t-1)}, P_{t}^{\xi}\right),\quad t\in\mathbb{T},
\end{equation*}
where $X_{t}^{\xi}$ is the agent's inventory and $P_{t}^{\xi}$ is the execution price. Here, $\xi_{0:(t-1)}$ denotes the agent's past actions between 0 and $t-1$, that is $\xi_{0:(t-1)}=(\xi_{0},\ldots,\xi_{t-1})$ for $t>0$ and it is the empty tuple for $t=0$. 
The state space $\mathcal{S}$ contains all the possible states in which the market is allowed to be at any given time as dictated by the market impact model defined in Section \ref{sec:market_impact_model}.

An action, denoted by $\xi_{t}$, represents the specific trade  undertaken by the agent at time $t\in\mathbb{T}$. Given that the initial inventory $X_{0}$ is positive, it is natural to restrict the set of allowed actions $\xi_{t}$ to the interval $[-X_{t}^{\xi},0]$ for any $t\in\mathbb{T}$. To enforce the terminal inventory constraint, we take the agent's final action to liquidate the remaining inventory, that is $\xi_{T} = -X_T^{\xi}$. The setup just described can be summarised by a state dependent action space $\{\mathcal{A}(s_t)\}_{t\in\mathbb{T}}$ with $\mathcal{A}(s_t) = [-X_{t}^{\xi},0]$ for  $t\in\{t_{0},\ldots,t_{N-1}\}$ and $\mathcal{A}(s_T) = \{-X_T^{\xi}\}$.
A strategy composed of a sequence of actions $\{\xi_t \in \mathcal{A}(s_t)\}_{t\in\mathbb{T}} $ is an admissible strategy in the sense of Definition \ref{def-admissible-strategy}. 

The reward function plays a crucial role in shaping the agent's behavior. As discussed in Section \ref{sec-problem-formulation}, for a given action $\xi_{t}$ at state $s_t$ the agent incurs the costs defined in \eqref{eq-returns}. Therefore, we take the reward function to be
\begin{equation}\label{eq:reward_function}
    \mathscr{R}(s_{t}, \xi_{t}) :=-\frac{\left(P_{t+}^{\xi}\right)^{2}-\left(P_{t}^{\xi}\right)^{2}}{2G(0)}, \quad t\in\mathbb{T},
\end{equation}
with $s_{t}\in\mathcal{S}$ and $\xi_{t}\in\mathcal{A}(s_t)$. 
Here again, the minus sign is introduced by switching from costs to rewards. 
Note that the RHS of~\eqref{eq:reward_function} depends on the state $s_{t}$ and action $\xi_{t}$ through~\eqref{eq:execution_price}.

The transition probability function is the conditional probability $P(s_{t+1}\,|\,s_{t},\xi_{t})$ of transitioning from a current state $s_{t}$ to a new state $s_{t+1}$ after taking an action $\xi_{t}$. It is dictated by the market model described in Section \ref{sec-problem-formulation}.

We focus on strategies that deterministically prescribe a unique action in $\mathcal{A}(s)$ for each state $s\in\mathcal{S}$. We will often refer to these strategies as \textit{policies}. We define an episode $\tau$ to be a sequence of states and actions in the environment, $\tau = (s_{t_{0}}, \xi_{t_{0}}, s_{t_{1}}, \xi_{t_{1}}, ..., s_{t_{N}},  \xi_{t_{N}})$. Finally, each transition in an episode is represented by a tuple given by $(s,a,r,s')$ where $s$ is the state of the market at the beginning of the transition, $a$ is the action taken by the agent, $r= \mathscr{R}(s,a)$ is the reward associated with the transition and $s'$ is the new state.

\subsection{Deterministic Policy Gradient} \label{sec:dpg}
The Deterministic Policy Gradient (DPG) algorithm~\citep{deterministic_policy_gradient_silver} is the foundation for our Deep RL approach. The DPG algorithm is an actor-critic method that relies on \textit{policy evaluation} and \textit{policy improvement} to guide the RL agent towards an optimal strategy.

\noindent\textbf{Policy Evaluation.} This step assesses the performance of the current policy in achieving the agent's objective~\eqref{eq-optimal-execution-problem} by leveraging the \textit{Q-function}.
We define the \textit{return} at any time $t_n\in\mathbb{T}$ as the sum of future rewards
\begin{equation*}
R_{t_n} := \sum_{k=n}^{N}\mathscr{R}\left(s_{t_{k}},a_{t_{k}}\right).
\end{equation*}
% The Q-function is defined as the expected return given that the market is in state $s_{t_n}\in\mathcal{S}$, the agent takes action $a_{t_n}\in\mathcal{A}(\pi_{\text{obs}}(s_{t_n}))$, and then acts according to policy $\xi$,
The Q-function of a state $s_{t_n}\in\mathcal{S}$ and action $a_{t_n}\in\mathcal{A}(s_{t_n})$ under policy $\xi$ is the expected return when starting in $s_{t_n}$, taking action $a_{t_n}$ and following $\xi$ thereafter
\begin{align} \label{eq:q_function}
Q^{\xi}(s_{t_n},a_{t_n}) := \mathbb{E}_{\xi}\left[R_{t_n} \, | \, s_{t_n},a_{t_n} \right],
\end{align}
where $\mathbb{E}_{\xi}[\cdot]$ denotes the expected value of a random variable given that the agent follows policy $\xi$. 
We define the  \textit{optimal Q-function} as the Q-function for the optimal strategy $\xi^{*}$, i.e. $Q^{*}(s,a) := Q^{\xi^{*}}(s,a)$.

In the DPG algorithm, the critic learns the optimal Q-function $Q^{*}(s,a)$ through a suitable function approximator $Q_{\phi}(s,a):\mathcal{S} \times \mathcal{A}\to \mathbb{R}$ with a real parameter vector $\phi$. This is possible because the optimal Q-function satisfies the following Bellman equation \citep[Chapter 4]{Sutton_Barto},
\begin{multline}
\label{eq-bellman-equation}
Q^{*}(s,a) = \\ \mathbb{E}\left[\mathscr{R}(s,a) + \max_{a'\in\mathcal{A}\left(s'\right)} Q^{*}\left(s',a'\right)  \, \Big| \, s, a\right],
\end{multline}
with $s'\sim P(s'\, |\, s,a)$. 
% The Bellman equation allows the algorithm to estimate the optimal Q-function using an appropriate policy evaluation algorithm such as Q-learning~\citep{Watkins1992}.  
Notice that the optimal action for any state $s\in\mathcal{S}$, denoted by $a^{*}(s)$, can be found by maximising $Q^{*}(s,a)$ over the action space $\mathcal{A}(s)$.
%The Q-function  can be expressed as a Bellman equation with,
%\begin{equation}
%Q^{*}(s,a) = \mathbb{E}_{s'\sim P(\cdot |s,a)}\left[r(s,a) + \max_{a'}Q^{*}(s',a')\right].  \label{eq:optimal_q_function}
%\end{equation}
% A fundamental connection exists between the actions prescribed by the optimal policy~\eqref{eq:optimal_policy} and the optimal Q-function~\eqref{eq:optimal_q_function}.
% The optimal Q-function $Q^{*}(s,a)$ gives the expected return for starting in state $s$, taking (arbitrary) action $a$, and then acting according to the optimal policy. Therefore the optimal action in state $s$, $a^*(s)$, is the action that maximizes the optimal Q-function value:
% \begin{equation} \label{eq:optimal_action}
% a^*(s) = \arg \max_a Q^* (s,a).
%\end{equation}

\noindent\textbf{Policy Improvement.} In the DPG algorithm,  we consider a deterministic policy $\xi_{\theta}:\mathcal{S} \to \mathcal{A}$ with a real parameter vector $\theta$ which approximates the optimal action $a^{*}(s)$ for any state $s\in\mathcal{S}$. Through policy improvement the agent refines its policy to maximize the expected execution profits $V^{\xi_{\theta}}$ defined in~\eqref{eq-expected-costs}. 
By adjusting the parameters $\theta$ in the direction of the objective gradient $\nabla_{\theta} V^{\xi_{\theta}}$, the performance of the policy $\xi_{\theta}$ can be repeatedly improved. 

As shown in \citet{deterministic_policy_gradient_silver}, the gradient $\nabla_{\theta} V^{\xi_{\theta}}$ can be expressed in closed form in terms of the Q-function~\eqref{eq:q_function} under the policy $\xi_{\theta}$. Let us denote by $P(s_{t_0}\to s', k,\xi)$ the probability of going from state $s_{t_0}$ to state $s'$ in $k$ steps under policy $\xi$. Moreover, let $\rho^{\xi}(s'):= \sum_{k=1}^{N+1}P(s_{t_0}\to s', k,\xi) $ be the (improper) marginal distribution of $s'$ given policy $\xi$. Lastly, denote by $\mathbb{E}_{s\sim\rho^{\xi}}[\cdot]$ the (improper) expected value with respect to the state distribution $\rho^{\xi}$. 
% Note that, using these definitions, it is straightforward to show that the objective of \eqref{eq-expected-costs} can be rewritten as $V^{\xi} = \mathbb{E}_{s\sim\rho^{\xi}}[\mathscr{R}(s,\xi(s))]$. 
As shown in \citet[Theorem 1]{deterministic_policy_gradient_silver}, the gradient of $V^{\xi_{\theta}}$ with respect to the policy's parameters $\theta$ is given by
\begin{multline}
\label{eq-policy-gradient}
\nabla_{\theta} V^{\xi_{\theta}} = \\ \mathbb{E}_{s\sim\rho^{\xi_{\theta}}}\left[\nabla_{\theta}\xi_{\theta}(s)\nabla_{a}Q^{\xi_{\theta}}(s,a)\big|_{a=\xi_{\theta}(s)}\right].
\end{multline}
% This method can be extended As shown in \citet[Theorem 1]{deterministic_policy_gradient_silver} 
In the context of our paper, we estimate the policy gradient \textit{off-policy} from trajectories sampled from a distinct behavior policy $\beta(s) \neq \xi_{\theta}(s)$. Following Section 4.2 of \citet{deterministic_policy_gradient_silver}, we modify the agent's objective to be 
\begin{equation*}
V_{\beta}^{\xi_{\theta}}  = \int_{\mathcal{S}} \rho^{\beta}(s) Q^{\xi_{\theta}}\left(s, \xi_{\theta}\left(s\right)\right) ds.
\end{equation*}
As discussed in \citet[Section 4.2]{deterministic_policy_gradient_silver}, the gradient of $V_{\beta}^{\xi_{\theta}}$ with respect to the policy's parameters $\theta$ can be approximated by
\begin{multline*}
\nabla_{\theta} V_{\beta}^{\xi_{\theta}} \approx \\\mathbb{E}_{s\sim\rho^{\beta}}\left[\nabla_{\theta}\xi_{\theta}(s)\nabla_{a}Q^{\xi_{\theta}}(s,a)\big|_{a=\xi_{\theta}(s)}\right].
\end{multline*}

% where $Q_{\phi}$ is a suitable function approximator for the optimal Q-function.
%We remark that, even though the state distribution $\rho^{\xi_{\theta}}$ depends on the policy parameters, the policy gradient \eqref{eq-policy-gradient} does not depend on the gradient of the state distribution.
% Placeholder: proof that the polivy gradients hold in our setup, i.e. regularity confidtions are satsfied. it should be rho of beta because we are in an off-policy set up.

\noindent\textbf{A Trick To Help Convergence.} 
We can simplify the optimization problem and facilitate the convergence of the algorithm by considering an \textit{auxiliary Q-function} (see Section \ref{sec:convergence_experiment}). 
%As shown in Figure~\ref{fig-auxiliary-Bellman}, experiments demonstrate that using the optimal auxiliary Q-function instead of the standard optimal Q-function improves the algorithm's convergence speed.
For the discussion that follows, it is convenient to define the projection operator $\hat{\pi}$ which projects any state $s_t\in\mathcal{S}$ to a 3-tuple corresponding to the first three elements of the state $s_{t}$, that is $\hat{\pi}(s_t) :=\left(t,X_{t}^{\xi},\xi_{0:(t-1)}\right)$. We adopt $\hat{\pi}(\mathcal{S})$ to denote the space of the projections of all the states in $\mathcal{S}$.
Let us define the {auxiliary Q-function}, $\mathcal{Q}^{\xi}:\hat{\pi}(\mathcal{S})\times\mathcal{A} \to \mathbb{R}$, with
\begin{equation}
\label{eq-auxiliary-q-function}
    \mathcal{Q}^{\xi}(\hat{\pi}(s), a) := Q^{\xi}(s, a) - x \cdot p_0,
\end{equation}
for $s\in\mathcal{S}$, $a\in\mathcal{A}(s)$ and $\xi\in\mathscr{X}(\mathbb{T},X_{0})$. Here, $x$ is the agent's inventory corresponding to state $s$ and $p_{0}$ is the initial price of the risky asset as defined in Section \ref{sec-problem-formulation}. 
It is straightforward to show, using~\eqref{eq:q_function} and~\eqref{eq-auxiliary-q-function}, that the auxiliary Q-function only depends on the restricted state space identified by $\hat{\pi}$.
This reduces the complexity of the input space in two ways, (1) by reducing the number of inputs from four to three, and (2) by considering inputs that are deterministic in nature. 
Henceforth, we will also consider the Q-function and optimal action approximators whose inputs are restricted by the projection operator $\hat{\pi}$, that is $Q_{\phi}(\hat{\pi}(s), a) : \hat{\pi}(\mathcal{S}) \times \mathcal{A} \to \mathbb{R}$  and $\xi_{\theta}:\hat{\pi}(\mathcal{S}) \to \mathcal{A}$.

Let us define the \textit{optimal auxiliary Q-function} as the auxiliary Q-function for the optimal strategy $\xi^{*}$, that is $\mathcal{Q}^{*}(\hat{\pi}(s),a) := \mathcal{Q}^{\xi^{*}}(\hat{\pi}(s),a)$. 
Notice that $Q^{*}(s,\cdot)$ and $\mathcal{Q}^{*}(\hat{\pi}(s),\cdot)$ share the same optimal action $a^{*}(\hat{\pi}(s))$ for any state $s\in\mathcal{S}$.
We wish to show that the auxiliary Q-function satisfies a Bellman equation as well as a policy gradient theorem. These two steps guarantee that $\mathcal{Q}^{\xi}$ can replace $Q^{\xi}$ in the DPG algorithm and that an optimal policy can be learnt from it. Our first result shows that the auxiliary Q-function also satisfies a Bellman equation.
\begin{proposition}[Auxiliary Bellman Equation]
\label{prop-q-learning}
Let $\mathcal{Q}^{*}$ be the optimal auxiliary Q-function. Then, $\mathcal{Q}^{*}$ satisfies the following Bellman equation
\begin{multline} \label{eq-auxiliary-bellman}
\mathcal{Q}^{*}(\hat{\pi}(s),a) = \mathbb{E}\Big[\mathscr{R}(s,a)+ a\cdot p_0\\   +  \max_{a'\in\mathcal{A}\left(s'\right)}\mathcal{Q}^{*}\left(\hat{\pi}\left(s'\right),a'\right) \, \Big| \, s, a \Big],
\end{multline}
with $s'\sim P(s' \, | \, s ,a)$ and for any $s\in\mathcal{S}$ and $a\in\mathcal{A}(s)$.
\end{proposition}
The proof of Proposition \ref{prop-q-learning} is postponed to Appendix \ref{app-proof}. In Algorithm~\ref{alg_ddpg}, the critic will exploit the Bellman equation \eqref{eq-auxiliary-bellman} to estimate the optimal auxiliary Q-function $\mathcal{Q}^{*}$ through a suitable function approximator $\mathcal{Q}_{\phi}$ with parameters $\phi$ (see \eqref{eq-msbe-loss}). We now turn to our next result, which follows immediately from \eqref{eq-policy-gradient} and \eqref{eq-auxiliary-q-function}.
\begin{proposition}[Auxiliary Policy Gradient]
\label{prop-policy-gradient}
Let $\mathcal{Q}^{\xi}$ be defined as in \eqref{eq-auxiliary-q-function}. Then,
\begin{multline*}
\nabla_{\theta} V^{\xi_{\theta}} = \\ \mathbb{E}_{s\sim\rho^{\xi_{\theta}}}\left[\nabla_{\theta}\xi_{\theta}(\hat{\pi}(s))\nabla_{a}\mathcal{Q}^{\xi_{\theta}}\left(\hat{\pi}\left(s\right),a\right)\big|_{a=\xi_{\theta}\left(\hat{\pi}\left(s\right)\right)}\right].
\end{multline*}
\end{proposition}
Similarly, it holds that
\begin{multline*}
\nabla_{\theta} V_{\beta}^{\xi_{\theta}} \approx\\ \mathbb{E}_{s\sim\rho^{\beta}}\left[\nabla_{\theta}\xi_{\theta}(\hat{\pi}(s))\nabla_{a}\mathcal{Q}^{\xi_{\theta}}\left(\hat{\pi}\left(s\right),a\right)\big|_{a=\xi_{\theta}\left(\hat{\pi}\left(s\right)\right)}\right],
\end{multline*}
which can be used to estimate the policy gradient in an off-policy fashion. 
% The proof of Proposition \ref{prop-policy-gradient} is deferred to Appendix \ref{app-proof}.
% By replacing ${Q}_{\phi}$ with $\mathcal{Q}_{\phi}$ in the agent's objective presented in~\eqref{eq:performance_objective_off_policy}, it is straightforward to show that the gradient of $V_{\beta}^{\xi_{\theta}}$ with respect to the policy's parameters $\theta$ is approximated with
% \begin{equation}
% \label{eq-policy-gradient-auxiliary}
% \nabla_{\theta} V_{\beta}^{\xi_{\theta}} \approx \mathbb{E}_{s\sim\rho^{\beta}}\left[\nabla_{\theta}\xi_{\theta}(s)\nabla_{a}\mathcal{Q}^{\xi_{\theta}}(s,a)\big|_{a=\xi_{\theta}(s)}\right].
% \end{equation}
% for any $s\in\mathcal{S}$, $a\in\mathcal{A}(s)$.

% \textcolor{red}{Explain why is called auxiliary Q-function.}
% \textcolor{red}{This is the on-policy gradient, you are actually using the off-policy one, discuss how they are related.}

\subsection{Deep Deterministic Policy Gradient} 
\label{sec-deep-dpg}
In RL, a key challenge is evaluating the Bellman equation to find the optimal action for a given state. In algorithms with discrete action spaces, e.g. the DDQN algorithm, the agent computes the Q-function for each available action by using the Bellman equation and chooses the one with the highest value by performing a maximization over a finite set. 
The situation becomes more complex when dealing with continuous action spaces as the number of possible actions is infinite and exhaustively evaluating the Q-function for every element within this continuous space is computationally impractical. 
The DPG algorithm described in Section~\ref{sec:dpg} addresses the challenges of continuous action spaces by learning a parameterized representation of the optimal policy, i.e. the sequence of optimal actions, and the optimal Q-function. Building upon the DPG algorithm, the DDPG algorithm~\citep{deep_dpg_lillicrap} approximates the optimal policy and the optimal Q-function with deep neural networks. We summarise the training scheme in Algorithm~\ref{alg_ddpg} and describe it in more details below.

\begin{algorithm}[p!]
    \caption{Deep Deterministic Policy Gradient for Optimal Execution Strategies}
    \label{alg_ddpg}
\begin{algorithmic}[1]
    \STATE Input: Initialize actor network's parameters $\theta$ and critic network's parameters $\phi$.
    \STATE Set target networks' parameters equal to policy networks' parameters $\theta^{\text{target}} \leftarrow \theta$, $\phi^{\text{target}} \leftarrow \phi$
    \FOR{$H$ episodes}
        \STATE Initialize state $s_{t_0}$.
        \STATE Initialize OU noise process $\varepsilon_0(\theta_{\varepsilon}, \sigma_{\varepsilon})$ and set $i = 0$.
        \FOR{$n = 0, ..., N$}
            \IF{$t_n < T$}
                \STATE \textbf{if} rand(1) $\leq \varepsilon$: Execute noisy action,
                \begin{equation*}
                % \quad\quad\quad\quad\quad a_{t_n} = -X_{t_n}^{\xi_{\theta}}\cdot\sigma\left(\xi_{\theta}\left(\hat{\pi}\left(s_{t_n}\right)\right) + \varepsilon_i\left(\theta_{\varepsilon}, \sigma_{\varepsilon}\right)\right), % for acm
                \quad\quad a_{t_n} = -X_{t_n}^{\xi_{\theta}}\cdot\sigma\left(\xi_{\theta}\left(\hat{\pi}\left(s_{t_n}\right)\right) + \varepsilon_i\left(\theta_{\varepsilon}, \sigma_{\varepsilon}\right)\right),
                \end{equation*}
            and update $i \leftarrow i + 1$. 
                \STATE \textbf{else}: Execute action,
                \begin{equation*}
                a_{t_n} = - X_{t_n}^{\xi_{\theta}}\cdot \sigma\left(\xi_{\theta}\left(\hat{\pi}\left(s_{t_n}\right)\right)\right).
                \end{equation*}
            \ELSE
                \STATE  Execute action, $a_{T} = -X_{T}^{\xi_{\theta}}$.
            \ENDIF
            \STATE Observe reward $r_{t_n}$ and next state $s_{t_{n+1}}$. If $s_{t_{n+1}} $ is terminal, set $d_{t_n} = 1$, otherwise set $d_{t_n} = 0$.
            \STATE Store episode $(s_{t_n}, a_{t_n}, r_{t_n}, s_{t_{n+1}}, d_{t_{n}})$ in memory buffer $\mathcal{M}$.
            \IF{$|\mathcal{M}| \geq D$} 
                \STATE Create experience replay buffer $\mathcal{D}$ by selecting the most recent $D$ episodes from $\mathcal{M}$.
                \STATE Randomly sample batch $\mathcal{B} = \{ (s, a, r, s', d) \}$ of size $B$ from $\mathcal{D}$.
                \STATE Compute target Q-function, $y(a, r,s',d)$, for each transition $(s, a, r, s', d)$ in $\mathcal{B}$ as shown in Equation~\eqref{eq:target_q_function}.
                \STATE Update critic network by one step of gradient descent using
                \begin{equation*}
                % \quad\quad\quad\quad\quad\nabla_{\phi} \frac{1}{B}\sum_{\substack{(s,a,r,  s',d)\\
                %   \in \mathcal{B}}} \left( \mathcal{Q}_{\phi}\left(\hat{\pi}(s),a\right) - y\left(a, r,s',d\right) \right)^2 % for acm
                \quad\quad\nabla_{\phi} \frac{1}{B}\sum_{\substack{(s,a,r,  s',d)\\
                  \in \mathcal{B}}} \left( \mathcal{Q}_{\phi}\left(\hat{\pi}(s),a\right) - y\left(a, r,s',d\right) \right)^2
                \end{equation*}
                \STATE Update actor network by one step of gradient descent using
                \begin{equation*}
                \quad\quad- \nabla_{\theta} \frac{1}{B}\sum_{s \in \mathcal{B}}\mathcal{Q}_{\phi}\left(\hat{\pi}(s), \xi_{\theta}\left(\hat{\pi}\left(s\right)\right)\right)
                \end{equation*}
                \STATE Update target networks with
                \begin{align*}
                \phi^{\text{target}} &\leftarrow \tau_{\phi} \phi^{\text{target}} + (1-\tau_{\phi}) \phi \\
                \theta^{\text{target}} &\leftarrow \tau_{\theta} \theta^{\text{target}} + (1-\tau_{\theta}) \theta
                \end{align*}
            \ENDIF
            %\STATE Keep track of best actions by running an episode within non-noisy actions. 
            \IF{$s_{t+1}$ is not terminal}
                \STATE Set $s_{t_n} \leftarrow s_{t_{n+1}}$.
            \ELSE
                \STATE Exit.
            \ENDIF
        \ENDFOR
    \ENDFOR    
\end{algorithmic}
\end{algorithm}

\noindent\textbf{Actor and Critic Networks Architectures.}
The actor network, denoted by $\xi_{\theta}(\hat{\pi}(s))$, is a neural network with parameters $\theta$. It takes the projected state, $\hat{\pi}(s)$, as input and outputs an approximation of the optimal action $a^*(\hat{\pi}(s))$. The critic network, denoted by $\mathcal{Q}_{\phi}(\hat{\pi}(s),a)$ and parameterized by a real vector $\phi$, takes as input a pair $(\hat{\pi}(s),a)$ and outputs an approximation of the optimal auxiliary Q-function $\mathcal{Q}^{*}(\hat{\pi}(s),a)$.
The critic and actor networks are fully-connected neural networks (FCN) that use the rectified linear unit (ReLU) activation function. The actor network's output is transformed using the sigmoid function, denoted by $\sigma(\cdot)$, and multiplied by the negative remaining inventory (see line 9 in Algorithm~\ref{alg_ddpg}). This ensures that the transformed output is an action value within the action space.

\noindent\textbf{Main and Target Networks.}
Using the same network for both generating the target Q-function and updating the current Q-function leads to gradient oscillations~\citep{Hasselt_Ddqn}. Target networks are introduced for the target Q-function generation. The parameters of the critic and actor target networks are denoted by $\phi^{\text{target}}$ and $\theta^{\text{target}}$, respectively. The target networks are updated once per main network update with Polyak averaging (see line 20 in Algorithm~\ref{alg_ddpg}).
% \begin{align}
% \phi^{\text{target}} \leftarrow (1 - \tau_{\phi})\phi^{\text{target}} + \tau_{\phi} \phi, \\
% \theta^{\text{target}} \leftarrow (1 - \tau_{\theta})\theta^{\text{target}} + \tau_{\theta} \theta.
% \end{align}
% where $\tau_{\phi}$ and $\tau_{\theta}$ are hyperparameters between 0 and 1 (usually close to 0).

\noindent\textbf{Training scheme.}
A memory buffer, denoted as ${\mathcal M} = \{(s,a,r,s')\}$, stores all past transitions experienced by the agent. To enable experience replay, a subset of the most recent $D$ transitions is drawn from the memory buffer to form the experience replay buffer, denoted by ${\mathcal D}$. 
% The size of the buffer, $D$, represents a trade-off between stability and learning speed. A larger buffer provides a broader pool of experiences, promoting stable learning \textcolor{red}{at the cost of a slower learning process (something like that)}.Conversely, a very small buffer can lead to overfitting on recent experiences.
The training process does not begin until the memory buffer contains at least $D$ entries. 
In our setting, episodes where liquidation occurs before the terminal step ($T$) violate the problem's assumptions and are therefore excluded from the memory buffer.
The cost function to be minimized by the critic network is the following mean-squared Bellman error (MSBE) loss,
\begin{multline}
\label{eq-msbe-loss}
\mathcal{L}^{\text{critic}}(\phi, {\mathcal D}) =\\ \underset{(s,a,r,s') \sim {\mathcal D}}{{\mathbb{E} }}\left[
    \left( \mathcal{Q}_{\phi}\left(\hat{\pi}\left(s\right),a\right) -  y\left(a,r,s',d\right)\right)^2
    \right],
\end{multline}
where 
\begin{multline} \label{eq:target_q_function}
 y(a,r,s',d)= r +  a\cdot p_{0} \\ + \left(1-d\right) \cdot \mathcal{Q}_{\phi^{\text{target}}}\left(\hat{\pi}\left(s'\right), \xi_{\theta^{\text{target}}}\left(\hat{\pi}\left(s'\right)\right)\right),
\end{multline}
and $d$ is the binary done signal to indicate whether $s'$ is terminal.
The cost function to be minimized by the actor network is
\begin{multline*}
\mathcal{L}^{\text{actor}}(\theta, {\mathcal D}) = \\- \underset{s \sim {\mathcal D}}{{\mathbb{E}}}\left[ \mathcal{Q}_{\phi}\left(\hat{\pi}\left(s\right), \xi_{\theta}\left(\hat{\pi}\left(s\right)\right)\right) \right],
\end{multline*}
where the critic network parameters $\phi$ are treated as constants. 

As in~\citet{deep_dpg_lillicrap}, we use a Ornstein-Uhlenbeck (OU) noise process with parameters $\theta_{\varepsilon}$ and $\sigma_{\varepsilon}$ and denoted at step $i$ by $\varepsilon_i(\theta_{\varepsilon}, \sigma_{\varepsilon})$. The noise is added on the unscaled output of the actor network, i.e., before the sigmoid function is applied (see line 8 in Algorithm~\ref{alg_ddpg}). The exploration probability at each iteration is denoted by $\varepsilon$.

\noindent\textbf{State and Action inputs normalization.} The state input is $\hat{\pi}(s_t) = \left(t,X_{t}^{\xi},\xi_{0:(t-1)}\right)$. The past actions $\xi_{0:(t-1)}$ are implemented by using a vector of fixed size $N+1$ which is initialized with $0$ and filled with the executed actions as the episode unfolds. 
Furthermore, we normalize the inputs of the critic and actor networks. Specifically, the state projection $\hat{\pi}(s_t)$ at time $t\in\mathbb{T}$ is normalized as $\left(t / T, X^{\xi}_{t} / X_0, \xi_{0:(t-1)}/{X_0}\right)$. The action at time $t\in\mathbb{T}$ is also normalized as $a_{t} / X_0$. 

\section{Numerical Results}
\label{sec-numerical-experiment}
%This section delves into the numerical results obtained from our investigation of the DDPG agent introduced in Section~\ref{sec-deep-dpg}. Section~\ref{sec:experiment_set_up} details the experimental setup, including the simulated financial market environment and the specific tuning parameters configuration used for the DDPG agent. In Section~\ref{sec:convergence_experiment}, we demonstrate the DDPG agent's ability to approximate the optimal solution for the addressed problem. Lastly, Section~\ref{sec:online_experiment} presents an experiment that highlights the advantages of using a RL strategy that adapts online in a dynamic financial market environment. 

\subsection{Set-up for numerical experiments} \label{sec:experiment_set_up}
\textbf{Environment set-up.} In our numerical experiments, the DDPG agent is trading in the financial market described in Section~\ref{sec-problem-formulation} with the unaffected price process volatility $\sigma_{W}=0.0001$.  The initial price is set to $p_0 = 50$, the initial inventory to $X_0 = 10$, and the time horizon to $N = 9$, such that the optimal execution strategy consists of $10$ trades. 

\noindent\textbf{DDPG algorithm set-up.} The hyperparameters for the DDPG algorithm are selected through a grid search procedure, prioritizing configurations that yielded the best performance. 
The critic and actor neural networks used by the DDPG agent are trained using PyTorch version 2.2.1+cu118. The critic FCN has $L^{\text{critic}} = 14$ hidden layers and $n_h^{\text{critic}} = 64$ units per hidden layer. The actor FCN employs $L^{\text{actor}} = 10$ hidden layers with $n_h^{\text{actor}} = 54$ units per hidden layer. Both the critic and the actor networks use PyTorch's default weight initialization (Xavier~\citep{pmlr-v9-glorot10a} for weights, zero for biases). The Adam optimizer~\citep{kingma2017adam} is used to update the parameters of the actor and critic networks. The learning rate of the critic network's optimizer is set to $\ell^{\text{critic}} = 5\times 10^{-4}$ and the learning rate of the actor network's optimizer is set to $\ell^{\text{actor}} = 5\times 10^{-5}$. We set the update rate for the target critic and target actor networks to $\tau_{\phi} = \tau_{\theta} = 0.005$. The OU noise process drift parameter is set to $\theta_{\varepsilon} = 0.15$.

\begin{figure*}[t!]
\centering
\includegraphics[width=1\textwidth]{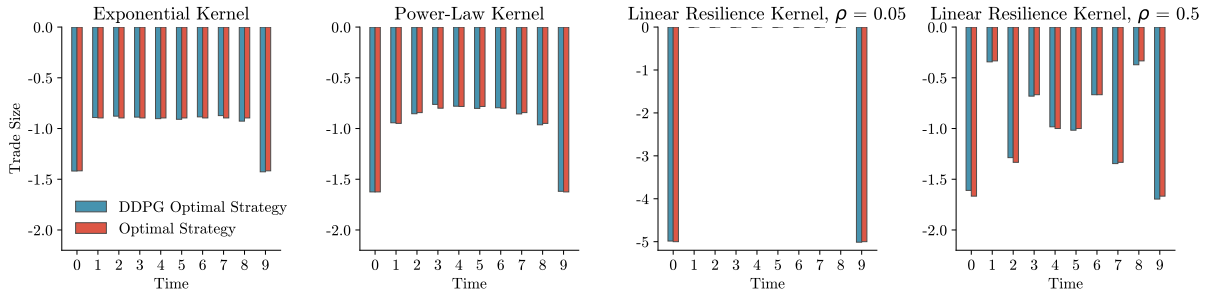}
\caption{ \textbf{Optimal Execution Strategies for Exponential Kernel, Power Law Kernel and Linear Resilience Kernel.} Blue bars represent the optimal strategy estimated by the DDPG agent and the red bars are the true optimal strategy (Equation~\eqref{eq:optimal_strategy}).}
\label{fig-convergence-strategies}
\end{figure*}

\subsection{Convergence to Optimal Execution Strategy} \label{sec:convergence_experiment}

\begin{figure}[b!]
\centering
\includegraphics[width=0.5\textwidth]{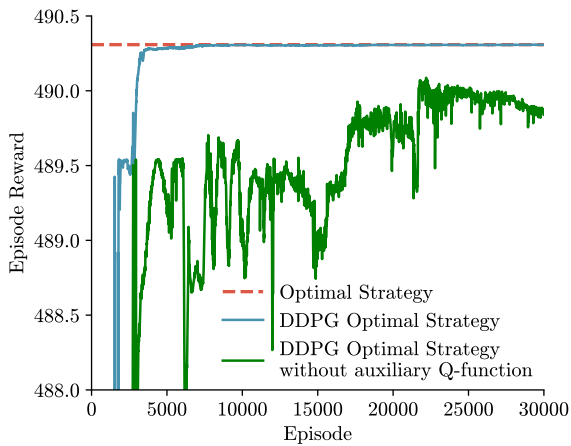}
\caption{\textbf{Advantage of Auxiliary Q-function.} Episode reward of the true optimal strategy (red), the optimal strategy estimated by the DDPG trained with the auxiliary Q-function~\eqref{eq-auxiliary-bellman} (blue) and trained with the standard Q-function~\eqref{eq-bellman-equation} (green). The experiment design is the same as that explained in Section~\ref{sec:convergence_experiment} for the exponential kernel.}
\label{fig-auxiliary-Bellman}
\end{figure}

The performance of the DDPG algorithm is tested on four different decay kernels, (1) the exponential kernel $G(t) = \kappa \exp(-\rho t)$ with $\kappa = \rho = 1$, (2) the power-law kernel $G(t) = \kappa (1+t)^{-\rho}$ with $\kappa = \rho = 1$, (3) the linear resilience kernel $G(t) = \kappa (1-\rho t)^{+}$ with $\kappa = 1$ and $\rho = 0.05$ and (4) the linear resilience kernel with $\kappa = 1$ and $\rho = 0.5$. The training for each decay kernel form was performed independently.

The DDPG agent undergoes training over $H = 30,000$ training episodes which corresponds to $H \times (N+1) = 300,000$ transitions. The replay batch size is set to $B = 1,000$ transitions. Different environments may necessitate varying replay buffer sizes for optimal performance~\citep{Zhang_2017}. In our experiments, we selected a replay buffer size of $D=15,000$ for the exponential and the power-law kernels, $D=10,000$ for linear resilience kernel with $\rho = 0.5$, and $D=20,000$ for linear resilience kernel with $\rho = 0.05$. Lastly, the probability of exploring and the OU noise process variance parameter are set to $\varepsilon=1$ and $\sigma_{\varepsilon} = 0.2$, respectively. 

Figure~\ref{fig-convergence-strategies} showcases the ability of the DDPG for Optimal Execution Strategies algorithm (Algorithm~\ref{alg_ddpg})  to estimate an optimal strategy (blue bars) aligning with the true optimal strategy (Equation~\ref{eq:optimal_strategy}) (red bars).
Here, we work under Assumption~\ref{assumption-convex-kernel} and we demonstrate the convergence across popular strictly convex kernels. 
In our experiments, we can determine the true optimal strategy because we have access to the exact decay kernel that has been used to generate the data. However, in real-world scenarios, the true decay kernel is typically unknown. To apply the closed-form solution in such cases, one must make assumptions about the parametric form of the decay kernel and estimate its parameters from market data. This approach aligns with the traditional methodology, as seen in works like ~\citet{Obizhaeva2013}. In contrast, our approach bypasses the need for such assumptions.

Figure~\ref{fig-auxiliary-Bellman} demonstrates that incorporating the auxiliary Q-function (introduced in Section~\ref{sec:dpg}) significantly improves the convergence of the DDPG algorithm. 
%todo say that the non auxiliary has the observanle state space
Specifically, using the exponential kernel as an example, we observe that the algorithm equipped with the auxiliary Q-function converges to the optimal strategy both faster and in a more stable manner.

\subsection{Online Learning in a Dynamic Environment} \label{sec:online_experiment}
The DDPG algorithm is capable of learning both within and across episodes.
In this experiment, we evaluate the DDPG agent's adaptation in a dynamic market environment. The experiment spans a period of $H=100,000$ episodes, during which the exponential kernel decay parameter, $\rho$, continuously changes\footnote{We note that the optimal strategy of \eqref{eq:optimal_strategy} is invariant under a change of the scale parameter $\kappa$ of the exponential kernel.}. We consider two scenarios where this parameter either linearly decreases from $1$ to $0.5$ or linearly increases from $1$ to $1.5$ (Figure~\ref{fig-dynamic-environment}a).
The experiment was performed independently for each scenario.

The DDPG agent leverages the accumulated memory buffer and the weights pre-trained during the convergence experiment (Section~\ref{sec:convergence_experiment}). The degree of exploration is controlled by the probability of exploring $\varepsilon=0.2$ and the OU noise process variance parameter $\sigma_{\varepsilon} = 0.14$. 

\begin{figure*}[t!]
\centering
\includegraphics[width=0.7\textwidth]{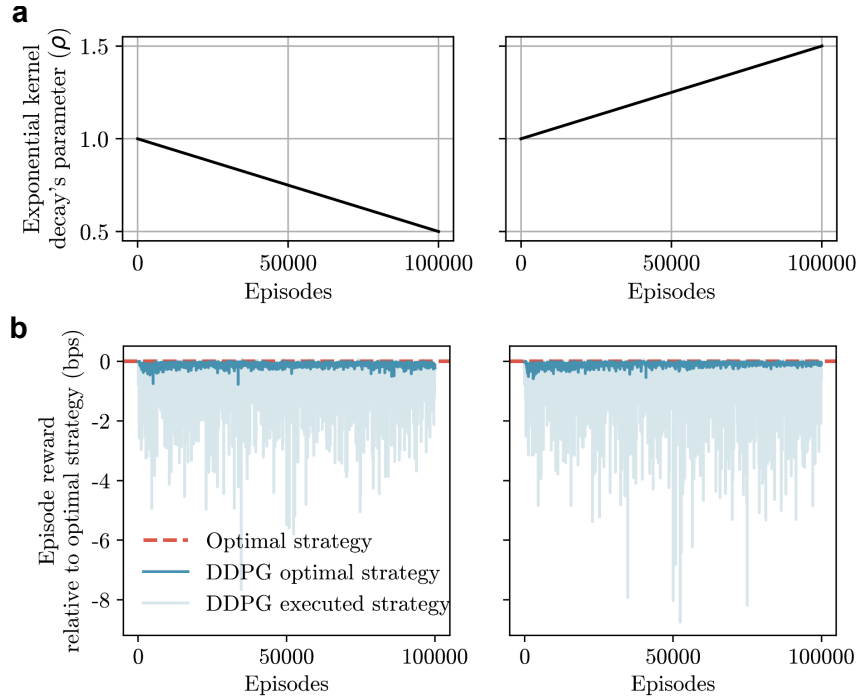}
\caption{\textbf{Online Learning in a Dynamic Environment.} (a) Exponential kernel decay parameter in a dynamic environment for two scenarios. (b) Episode reward of the estimated optimal strategy (dark blue) and the executed strategy (light blue) of the DDPG agent relative to the episode reward of the true optimal strategy (Equation~\eqref{eq:optimal_strategy}) (red).}
\label{fig-dynamic-environment}
\end{figure*}

In Figure~\ref{fig-dynamic-environment}b we measure the relative difference between the episode reward associated with the DDPG agent's strategy (light and dark blue lines) and those associated with the true optimal strategy (red line) in basis points (1 bps = 0.01\%) per episode. 
The light blue line represents the actual episode reward achieved by the DDPG agent's executed strategy. This includes the impact of exploration, where the agent tries non-optimal actions. In contrast, the dark blue line shows the hypothetical episode reward the DDPG agent could have achieved if it could exploit its learned optimal strategy at the end of each episode.
The light blue line closely tracks the optimal red line, showcasing the DDPG agent's ability to learn online. The small and stable deviation observed between the dark blue and the light blue line (executed strategy) is attributed to exploration noise.

\section{Conclusion}
\label{sec-conclusions}
This paper proposes a deep RL approach based on the DDPG algorithm to tackle the optimal execution problem. Our experiments demonstrate the effectiveness of DDPG in approximating both the Q-function and the continuous action policy, enabling efficient execution strategies. 
Through numerical experiments, our DDPG algorithm demonstrated the ability to learn the optimal execution strategy in environments governed by different decay kernels. Additionally, the algorithm showed a capacity to adapt online to the dynamic nature of real-world financial markets, simulated through varying kernel's parameters.
Our findings highlight that the algorithm's parameters are environment-dependent, and tuning may be required to achieve optimal performance in different settings.
Future work could extend the current model to encompass trading across a basket of stocks. Additionally, exploring the integration of exogenous signals into the optimal execution strategy holds potential for enhancing decision-making capabilities.

% We show that the DDPG can adjust its behavior in response to changing market conditions. 

% The DDPG's online learning capability offers a significant advantage. Unlike methods that rely solely on experience replay after each episode, our model can continuously learn and update its policy within each episode. This allows for more fine-grained adaptation to market fluctuations, potentially leading to improved performance.

% Our findings pave the way for the development of more sophisticated deep RL execution models that can navigate real-world market complexities. 
 
% This would necessitate capturing the intricate cross-impact effects between these assets, significantly increasing the model's real-world applicability. 

% This capability offers a significant advantage compared to standard agents that update their beliefs about the environment (e.g., kernel parameters) at discrete time intervals.


%
%
% ACKNOWLEDGMENTS
%
%

\section{Acknowledgments}
We acknowledge computational resources and support provided by the Imperial College Research Computing Service (http://doi.org/10.14469/hpc/2232).

%
%
% BIBLIOGRAPHY
%
%
\bibliographystyle{ACM-Reference-Format}
\bibliography{ref}

%
%
% APPENDIX
%
%

\appendix
%\counterwithin{figure}{section}

%\section{Figures}

\section{Proofs}
\label{app-proof}
\begin{proof}[Proof of Proposition \ref{prop-q-learning}]
Let $s\in\mathcal{S}$ and $a\in\mathcal{A}(s)$ and let $P(s'|s,a)$ be the transition probability function according to which the states evolve.
Recall that the optimal Q-function $Q^{*}$ satisfies the Bellman equation of \eqref{eq-bellman-equation}. By using the definition of $\mathcal{Q}^{*}$ from \eqref{eq-auxiliary-q-function} in \eqref{eq-bellman-equation} we obtain
\begin{multline}
\label{eq-proof-q-star-replaced}
\mathcal{Q}^{*}(\hat{\pi}(s),a) + x\cdot p_{0}= \\\mathbb{E}_{s,a}\Big[\mathscr{R}(s,a) + \max_{a'\in\mathcal{A}(s')} \left(\mathcal{Q}^{*}(\hat{\pi}(s'),a') + x'\cdot p_{0}\right)\Big]
\end{multline}
where $s'\sim P(s' \, | \, s,a)$ and $x$ and $x'$ are the agent's inventories corresponding to the states $s$ and $s'$, respectively.
Notice that the term $x'\cdot p_{0}$ can be pulled out of the maximization in the right-hand side of \eqref{eq-proof-q-star-replaced}.
By using doing so and then rearranging the terms we get
\begin{multline}
\label{eq-proof-bellman-before-substitution}
\mathcal{Q}^{*}(\hat{\pi}(s),a) = \\\mathbb{E}_{s,a}\Big[\mathscr{R}(s,a) +(x'-x)\cdot p_{0}\\ + \max_{a \in
\mathcal{A}(s')} \mathcal{Q}^{*}(\hat{\pi}(s'),a') \Big]
\end{multline}
As dictated by the market model in Section \ref{sec:market_impact_model}, when the state $s$ transitions to $s'$, the agent's inventory evolves deterministically based on \eqref{eq-agent-inventory}.  It follows immediately from \eqref{eq-agent-inventory} that $x'-x = a$
which can be plugged in \eqref{eq-proof-bellman-before-substitution} to obtain the desired Bellman equation for $\mathcal{Q}^{*}$. This completes the proof.
\end{proof}

\clearpage
\newpage

% \clearpage
% \newpage
% \section{Parameters for numerical experiments}
% \input{tables_experiment}

\end{document}